\documentclass[mlabstract]{jmlr}

\usepackage{longtable}% for long tables

\usepackage{booktabs}
\usepackage[load-configurations=version-1]{siunitx} % newer version
\usepackage{algorithm}
\usepackage{algpseudocode}
\usepackage{xcolor}
\usepackage{natbib}
\usepackage{amsmath}
\newcommand{\Ncal}{\mathcal{N}}
\newcommand{\reals}{\mathbb{R}}
\newcommand{\norm}[1]{\left\|#1\right\|}

\usepackage{wrapfig}

\theorembodyfont{\upshape}
\theoremheaderfont{\scshape}
\theorempostheader{:}
\theoremsep{\newline}

\jmlrvolume{}
\firstpageno{1}
\jmlryear{2024}
\jmlrworkshop{Symmetry and Geometry in Neural Representations}

\title[Short Title]{On the Reconstruction of Training Data from Group Invariant Networks}

\author{\Name{Ran Elbaz\textsuperscript{1}, Gilad Yehudai\textsuperscript{2},  Meirav Galun\textsuperscript{3} Haggai Maron\textsuperscript{1,4} \\}
\addr \textsuperscript{1} Technion – Israel Institute of Technology, Electrical and Computer Engineering \\
\textsuperscript{2} Center for Data Science, New York University \\
\textsuperscript{3} Weizmann Institute of Science \\
\textsuperscript{4} NVIDIA Research
}

\begin{document}

\maketitle

\begin{abstract}
    Reconstructing training data from trained neural networks is an active area of research with significant implications for privacy and explainability. Recent advances have demonstrated the feasibility of this process for several data types. However, reconstructing data from group-invariant neural networks poses distinct challenges that remain largely unexplored. This paper addresses this gap by first formulating the problem and discussing some of its basic properties. We then provide an experimental evaluation demonstrating that conventional reconstruction techniques are inadequate in this scenario. Specifically, we observe that the resulting data reconstructions gravitate toward symmetric inputs on which the group acts trivially, leading to poor-quality results. Finally, we propose two novel methods aiming to improve reconstruction in this setup and present promising preliminary experimental results. Our work sheds light on the complexities of reconstructing data from group invariant neural networks and offers potential avenues for future research in this domain. 
\end{abstract}

\section{Introduction}
\label{sec:intro}
\begin{wrapfigure}[13]{R}{0.35\textwidth}
  \begin{center}
  \label{fig: orbitopes}
  \caption{Visualization of orbitope with $G=D_4$}
    \includegraphics[width=0.3\textwidth]{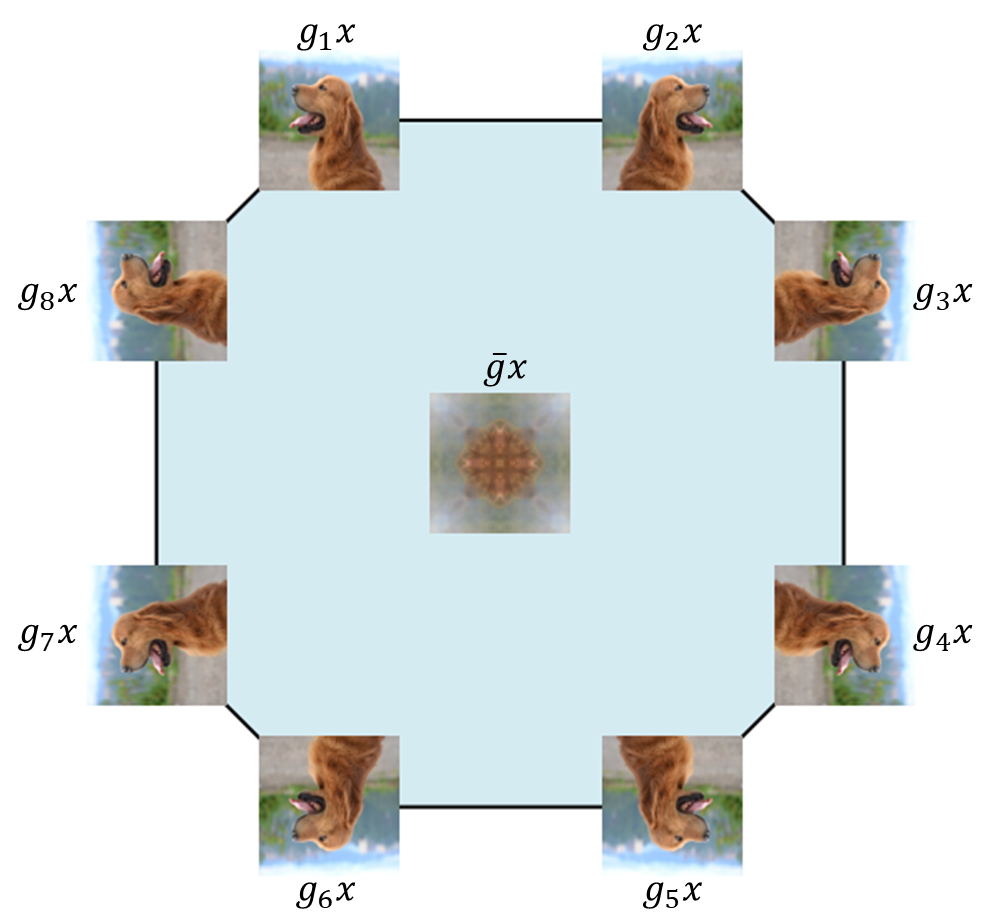}
\end{center} 
\end{wrapfigure}
Recent works \citep{haim2022reconstructing,oz2024reconstructing,loo2023dataset} have shown that it is possible to reconstruct training data from standard neural networks. However, the reconstruction from group invariant neural networks, such as networks applied to point clouds \citep{zaheer2017deep,qi2017pointnet}, graph data \citep{gilmer2017neural} or images with rotation and reflection symmetries \citep{cohen2016group}, remains largely unexplored. 

Unlike the standard case explored in previous works, reconstructing data from group invariant models faces the unique challenge of multiple distinct inputs representing the same data point (namely, the orbit of a data point). 
This paper addresses the task of reconstructing data from group invariant networks, focusing on the limitations of conventional methods and proposing novel solutions.
Our key contributions include:
\begin{enumerate}
    \item A formal definition of the reconstruction problem for invariant networks, accompanied by a discussion of its fundamental properties and invariance characteristics.
    \item Empirical evidence demonstrating that conventional reconstruction methods converge to symmetric inputs (i.e., inputs on which the group acts trivially), producing low-quality reconstructions. The inset illustrates this phenomenon for $G=D_4$: rather than recovering an element from a training example's orbit, conventional methods often converge to the (symmetric) orbit average.
    \item Introduction of two novel techniques that enhance standard methods, enabling them to go beyond symmetric reconstructions with encouraging initial results.

    \item A discussion of potential future research directions.
\end{enumerate}

\section{Preliminaries} \label{sec: Preliminaries}
\textbf{Invariance and equivariance.} Let $(V, \rho)$,$(V', \rho')$ be  group representations of a finite group $G$. We denote $orb(\Vec{x})$ as the orbit of a vector $\Vec{x} \in V$ under the group action and  $Stab(\Vec{x})$ as its stabilizer group. The orbitope of a point $\Vec{x}$ (\cite{Sanyal_2011}) is defined as the convex hull of $orb(\Vec{x})$ as illustrated in \figureref{fig: orbitopes}. We denote  the vectors on which the group acts trivially as $V^G$, formally $V^G = \{\Vec{v} \in V | \rho(g)\Vec{v}  = \Vec{v},  \forall g \in G \}$.
The projection of a vector $\Vec{x}$ on $V^G$, denoted as $\bar{g} x$, is the average of its orbit, $\bar{g} \cdot \Vec{x} =  \frac{1}{|G|}\sum_{g\in G} \rho(g) \cdot \Vec{x}$.  
A function $f:V \rightarrow V'$ is $G$-invariant if $f\circ \rho(g) = f$ for any $g\in G$ and a function $F: V\rightarrow V'$ is $G$-equivariant if $F\circ \rho(g) = \rho'(g) \circ F$  for any $g\in G$. For simplicity, we denote $\rho(g) \Vec{x} = g \Vec{x}$.

\noindent \textbf{Data Reconstruction.} \label{sec:Data Reconstruction} There are several methods to reconstruct training data from trained neural networks $\phi(\Vec{x};\theta)$, where $\vec{x}\in \reals^d$ is the network's input, and $\theta$ is a vectorization of its parameters. Here, we focus on two methods: (1) Activation Maximization (AM) \citep{fredrikson2015model,yang2019neural}, where the goal is to look for the input, which maximizes the model output for the desired target class. Namely, the objective for class $i$ is defined as $\mathcal{L}_{rec} = \max_{\Vec{x}\in\reals^d} (\phi(\Vec{x};\theta))_i$. (2) KKT-based reconstruction \citep{haim2022reconstructing,buzaglo2023deconstructing}. This method uses the implicit bias of homogeneous neural networks trained with gradient methods toward margin maximization. Here, the following objective is optimized:
$    \mathcal{L}_{rec}(\Vec{x}_1,\dots,\Vec{x}_m,\lambda_1,\dots,\lambda_m) = \norm{\theta - \sum_{i=1}^m\lambda_i y_i \phi(\Vec{x}_i; \theta)}$, (The $y_i's$ denote the labels.). For a detailed description of the reconstruction methods, see Appendix \ref{appendix:recon methods}.

vspace{-10pt}
\section{Reconstruction from Invariant networks}
\subsection{Problem definition} \label{sec: problem defintion}

Let $\set{D} = \{(\Vec{x}_i,y_i)\}_{i=1}^n\in\reals^d \times\{\pm 1\}$ be a training dataset and let $(\reals^d, \rho)$ is an orthogonal representation of a finite group $G$. Let $\phi: \reals^d \rightarrow \reals $ be a pre-trained $G$-invariant neural network with weights $\theta \in \reals^p$.
We aim to find a set of reconstructions $\set{S} = \{\hat{\Vec{x}}_1, \dots, \hat{\Vec{x}}_m \} \subset \reals^d$ such that $\set{S}$ is the closest set to $\set{D}$ with respect to some evaluation metric.

\noindent \textbf{Evaluation under group symmetries.} Since the reconstruction of any element from the orbit is equally valid, the evaluation metric for invariant data reconstruction problems should be invariant to these group symmetries to ensure a fair assessment of model performance. This can be done, for example, by defining the distance between a reconstruction $x$ and a training example $x'$ using the following metric $d(x,x')=\min_{g\in G} \|x-gx'\| $. Note that invariant metrics may be computationally challenging, e.g. when the input is a graph.

\noindent\textbf{Applying AM and KKT to the invariant case.} In most cases the training of (homogeneous) invariant neural networks is conducted in a way that the conditions of both methods (AM and KKT-based) are met, so we can apply them to the invariant case. 
\vspace{-5pt}
\subsection{Challenges and theoretical observations} 
\label{subsec: theoretical properties}
Here we present basic theoretical findings and challenges in reconstructing data from invariant models, applicable to any orthogonal representation of a finite group. Full proofs are in the appendix.

\noindent \textbf{(1) Multiple equivalent solutions.} When reconstructing data from invariant models, a critical factor to consider is that each training sample can have multiple equivalent representations. These representations form what is known as an orbit under the group action.  
\begin{proposition} \label{prop: invart loss}
    If the model is $G$-invariant then the objective functions of the methods mentioned in \sectionref{sec: Preliminaries} are $G$-invariant. Formally,
    \begin{equation}
        \mathcal{L}_{rec}(\Vec{x}_1,\dots,\Vec{x}_m) = \mathcal{L}_{rec}(g_1 \cdot \Vec{x}_1 ,\dots, g_m x_m),  ~ \forall (x_1 , g_1) , \dots , (x_m, g_m) \in \reals^d \times G
    \end{equation} 

\end{proposition}

\noindent \textbf{(2) Optimizing invariant reconstruction objectives using GD.}
As invariant reconstruction objectives have multiple optima, both initialization and optimization methods are crucial in determining the final solution. The following proposition sheds light on GD's behavior in this context:
\vspace{-5pt}
\begin{proposition} \label{lemma:trivial closere} If the reconstruction objective function $\mathcal{L}_{rec}(\Vec{x} ;\theta)$ is $G$-invariant function then: (i) the GD step  $x_t = x_{t-1} -\eta_t \nabla_x \mathcal{L}_{rec}(x_{t-1}) $ is  G-equivariant function of $x_{t-1}$;  and (ii) $\text{Stab}_G(x_{t-1}) \subseteq \text{Stab}_G(x_{t})$
\end{proposition}
First, the above part (i) implies that GD is an equivariant function of the initialization, as it is a composition of equivariant functions (GD iterations). Therefore, the initialization determines which element the method converges to.  Moreover, it implies that if we use invariant distribution for the initialization ($P(x_0)$ is a $G$-invariant function) the algorithm induces an invariant distribution over the reconstructions. 
Moreover, the nesting of the stabilizers mentioned in part (ii) of Proposition \ref{lemma:trivial closere} indicates that as optimization progresses, the stabilizers of the iterates may become more restrictive, thereby narrowing the exploration of the solution space. As we will see in the following sections, we believe that this property plays a significant role in the dynamics of optimization and can influence the final outcomes.

\vspace{-5pt}

\subsection{Ineffectiveness of standard methods in Invariant Reconstruction }
This subsection presents experimental evidence demonstrating the ineffectiveness of AM and KKT-based methods in solving the invariant reconstruction problem. The full experimental results and description are on the appendix.

\noindent \textbf{Setup and evaluation.} We focus on the reconstruction of image data from invariant models of different groups of reflections and rotations  (see \sectionref{appen: Experimental setting}).  To evaluate the results we used DSSIM proposed in \cite{baker2023dssimstructuralsimilarityindex} for measuring differences between images (high DSSIM implies high structural dissimilarity).To ensure the invariance of our metric, we match reconstructions to training samples across all group transformations.

\noindent \textbf{Results.}  As illustrated in Figures \ref{fig:am results} and \ref{fig:kkt results}, the optimization often converges to invariant reconstructions, resulting in a significant loss of information and low-quality reconstructions. Moreover, our observations reveal that many reconstructions lie on the convex hull of sample orbits, or orbitopes. To understand the distribution of reconstructions on the orbitopes we sampled various points on the ground truth orbitopes and identified the nearest neighbors for each reconstruction.  as depicted in Figure \ref{fig:dssim kkt} the observed distribution aligns with our predictions in Section \ref{subsec: theoretical properties}: the reconstructions are concentrated around the group average that lies in $V^G$ with the largest possible stabilizer.
Notably, the KKT-based method outperforms activation maximization, as shown in Figure \ref{fig: dssim}. Furthermore, we observe that increasing either the group size or the training set size leads to poorer results. %These findings corroborate the results reported by \citet{haim2022reconstructing} and \citet{buzaglo2023deconstructing}. \haggai{only the finding about data setsize no?}

\vspace{-10pt}
\section{Symmetry-aware reconstruction}
We propose two methods to improve reconstruction: mitigating GD's bias towards symmetry and imposing meaningful input space priors.

\noindent
\textbf{Symmetry-Aware Memory-Enhanced Gradient Descent (SAME-GD).}
Empirically, we tend to converge to points with nontrivial stabilizers, in particular to points on or close to $V^G$. We suggest aggregating the current query point with previous points in the optimization trajectory in a way that breaks the nesting property proved in Proposition \ref{lemma:trivial closere}. For simplicity, we used convex aggregation in the form of $\Vec{x}_t \leftarrow \alpha_t \Vec{x}_t + (1-\alpha_t) (\Vec{x}_{prev} - \bar{g} \Vec{x}_{prev})$, see \algorithmref{alg:sam}.

\noindent
\textbf{Incorporating Deep Image Prior (DIP).}
As proposed in \cite{Ulyanov_2020}, convolutional neural networks can be used as an implicit prior when it comes to inverse problems. We propose to use the same objective functions of the existing methods, but parameterizing the reconstruction variables $\vec{x}_1,\dots, \Vec{x}_m$ as the output of a randomly initialized CNN instead of optimizing them directly. The motivation is that the natural image prior could potentially break the symmetry and enhance the quality of the reconstructions.

\noindent
\textbf{Preliminary experimental results.}
We investigated the reconstruction abilities of the proposed methods under different configurations as listed on \tableref{table: dssim}. We extended our experimental results to include CIFAR-10 images, with binary labels indicating animals and vehicles. %The reconstructions vary significantly depending on the applied method, highlighting that the choice of reconstruction technique is crucial for achieving high quality results.
%The Activation Maximization approach often produces blurred reconstructions that lay in $V^G$. On the other hand, 
SAME-GD and DIP, when combined with the KKT objective, yield notably improved reconstructions. These methods show a reduced tendency to converge to group averages, thus preserving more meaningful data characteristics. Particularly noteworthy is the KKT with DIP approach (Figure \ref{fig: results}), which excels in producing piece-wise smooth asymmetric reconstructions by exploiting the implicit prior induced by DIP.

\begin{figure}[htbp]
\floatconts
  {fig: results}
  {\caption{Pairs of training samples and their corresponding nearest neighbors reconstructions on their left, where $n=50,|G|=2$ is the group of right-left reflections.}}
  {%
    \begin{tabular}{p{0.25\linewidth}p{0.75\linewidth}}
      AM & \includegraphics[width=0.95\linewidth]{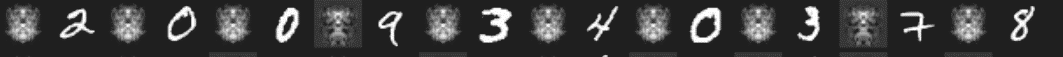} \\
      KKT & \includegraphics[width=0.95\linewidth]{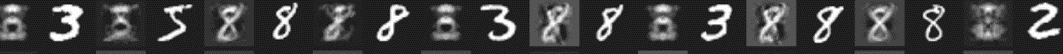} \\
      KKT + SAME-GD & \includegraphics[width=0.95\linewidth]{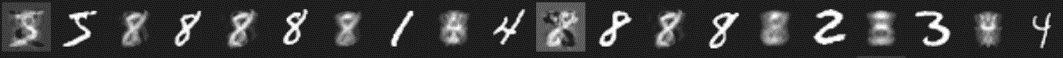} \\
      KKT + DIP & \includegraphics[width=0.95\linewidth]{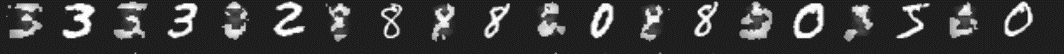}
    \end{tabular}
  }
\end{figure}

% \vspace{-10pt}
\begin{table}[h!] \label{table: dssim}
\tiny
\centering
\begin{tabular}{lccccccc}
\toprule
Dataset & Group size & Train set size & AM & KKT & KKT + SAME-GD & KKT + DIP \\
\midrule
MNIST & 2 & 50  & $0.482 \pm 0.000 $ & $0.464 \pm 0.000$ & $0.429 \pm 0.004$ & $ \mathbf{0.285 \pm 0.018}$ \\
MNIST & 2 & 100 & $0.484 \pm  0.000 $ & $0.467 \pm 0.000$ & $0.450 \pm 0.004$ & $\mathbf{0.271 \pm 0.000}$ \\
CIFAR-10 & 2 & 50  & $0.446 \pm  0.000 $ & $\mathbf{0.346 \pm 0.001}$ & $0.369 \pm 0.000$ & -- \\
CIFAR-10 & 2 & 100 & $0.463 \pm 0.000 $ & $0.371 \pm 0.001$ & $\mathbf{0.370 \pm 0.007}$ & -- \\
MNIST & 8 & 50  & $0.490 \pm 0.002$ & $0.471 \pm 0.000$ & $0.465 \pm 0.000$ & $\mathbf{0.314 \pm 0.021} $ \\
MNIST & 8 & 100 & $0.494 \pm 0$ & $0.471 \pm 0.001$ & $0.469 \pm 0.000$ & $ \mathbf{0.323 \pm 0.025}$\\
\bottomrule
\end{tabular}
\caption{The Mean DSSIM value for different methods across datasets, group sizes, and training set sizes.}
\end{table}
\vspace{-20pt}
% \\
\noindent
\textbf{Discussion.}
Our work highlights the challenges in reconstructing training data from group-invariant neural networks. The theoretical and experimental foundations laid here raise questions about the behavior of reconstruction methods applied to these networks. 
Although we provide some insights and novel approaches, It is still unclear why standard reconstruction methods fail for invariant models. 
There are many future directions to be explored, some of them are discussed in more details in Appendix \ref{appendix:future directions}.

\acks{HM is the Robert J. Shillman Fellow, and is supported by the Israel Science Foundation through a personal grant (ISF 264/23) and an equipment grant (ISF 532/23). We would like to thank Yam Eitan for his valuable contribution to this project.}

\bibliography{pmlr-sample}

\appendix

\section{Previous work}
 Several methods have been developed to reconstruct training samples from neural networks under different settings. Activation-maximization attacks \cite{fredrikson2015model,yang2019neural} optimize the target output class over the input. Another method is reconstruction in a federated learning setup \cite{zhu2019deep,hitaj2017deep,geiping2020inverting,huang2021evaluating} where the attacker is assumed to have knowledge of the sample's gradient. Several works use the implicit bias of neural networks towards margin maximization \cite{lyu2019gradient,ji2020directional} to devise reconstruction losses, and thus reconstruct samples that are on the margin \cite{haim2022reconstructing,buzaglo2023deconstructing,loo2023dataset,oz2024reconstructing}.
Some prior studies have explored reconstructing graph structures from trained networks. The majority of existing research has concentrated on single-graph learning scenarios. In these cases, known node feature matrices effectively break symmetries, which simplifies the problem \cite{zhang2021graphmi, wu2021adapting}.

\section{Current reconstruction methods}\label{appendix:recon methods}
This is elaboration of \sectionref{sec:Data Reconstruction} in the main text. There are several methods to reconstruct training data from trained neural networks. In this work we focused on two methods which allow data reconstruction in a general setting with minimal assumptions on the model's architecture.

\noindent
\textbf{Activation-Maximization (AM).}
We are given a trained multi-class classifier $\phi:\reals^d\rightarrow\reals^C$ with $C$ classes. The predicted class of the classifier is defined as $\max_{i\in[C]}(\phi(\Vec{x}))_i$, namely, the class with the maximal output. In this reconstruction method, to reconstruct a sample in class $i$, we randomly initialize an input $\Vec{x}\sim\Ncal\left(0,\frac{1}{d}I\right)$ and maximize the loss objective  $\mathcal{L}_{rec} = \max_{\Vec{x}\in\reals^d} (\phi(\Vec{x}))_i$. This is done by applying a first-order optimization method such as (Gradient Descent) GD.

\noindent
\textbf{KKT-based reconstruction.}
\cite{lyu2019gradient,ji2020directional} show that given a homogeneous\footnote{A function $f$ is $L$ homogeneous if for every $\alpha > 0$ we have $f(\alpha \Vec{x}) = \alpha^Lf(\Vec{x})$} neural network $\phi(\cdot,\theta)$, trained with gradient flow using an exponentially tailed loss (e.g., binary cross entropy) on a binary classification dataset $\{(\Vec{x}_i,y_i)\}_{i=1}^n\in\reals^d \times\{\pm 1\}$, its parameters $\theta$ converge to a KKT point of the following margin maximization problem
\begin{equation}
    \min \norm{\theta}^2~~~ \text{s.t.}~~ \forall i=1,\dots,n,~~y_i\phi(\Vec{x}_i, \theta) \geq 1.
\end{equation}
In particular, the KKT stationary condition is satisfied, namely there exist $\lambda_i \geq 0$ for $i=1,\dots,n$ such that $\theta = \sum_{i=1}^n \lambda_i y_i \phi(\Vec{x}_i,\theta)$. In 
\cite{haim2022reconstructing} 
the authors use the stationary condition to construct an objective that reconstructs the training data $\Vec{x}_i$ given the trained weights $\theta$. Namely, they propose optimizing the following loss objective
\begin{equation}
    \mathcal{L}_{rec}(\Vec{x}_1,\dots,\Vec{x}_m,\lambda_1,\dots,\lambda_m) = \norm{\theta - \sum_{i=1}^m\lambda_i y_i \phi(\Vec{x}_i, \theta)} 
\end{equation}
where $m$, the number of reconstruction candidates is chosen to be $m \gg n$. This optimization problem in practice is solved by GD or similar optimization methods.

\section{Proof of proposition \ref{prop: invart loss}} \label{appendix: proof for prop1}
The activation maximization objective function uses the model output directly. Since the model is invariant it is trivially implying that the objective loss is also invariant.
As the KKT based method involves first-order derivatives we would start by proving the following lemma:
\begin{lemma} \label{lemma:invariant gradient} Let $f(x ; \theta):\reals^d \times \reals^p \rightarrow \reals $ be $G$-invariant function w.r.t $x$. Assume $f$ has a partial gradient by $\theta$ at $(x_0 , \theta_0)$ . Then $f(x;\theta)$ also has partial gradient by $\theta$  at $\left\{(g\cdot x_0 , \theta_0)\right\}_{g \in G}$ . Moreover, $\nabla_\theta f(x_0;\theta) = \nabla_\theta f(g \cdot x_0;\theta) , \forall g \in G$
 \end{lemma}
 \begin{proof}
 Denote $\{e_1, e_2, ..., e_p\}$ to be the standard basis of $\reals^p$.
f is $G$-invariant, therefore $\forall i= 1,\dots,p,  \forall \epsilon \in \reals , \forall g\in G$ ,\begin{equation}
      \frac{ f(x;\theta + \epsilon e_i) - f(x;\theta) }{ \epsilon} = \frac{f(g \cdot x;\theta + \epsilon e_i) - f(x;\theta) }{ \epsilon} 
 \end{equation}
 Since it is given the limit of $\epsilon \rightarrow 0$ exists for the left side of the equation then the limit of the right side also exists and is equal to it.
 If we take the limit of both side, by definition we get:
 \begin{equation}
     \frac{\partial f(x;\theta)}{\partial \theta_i} = \frac{\partial f(g \cdot x;\theta)}{\partial \theta_i}
 \end{equation}
\end{proof}

In other words $\nabla_\theta f(\cdot, \theta)$ is $G$-invariant. As the trained model $\phi$ is invariant, we can say that $\nabla_\theta \phi $ is also invariant and therefore the objective loss in the KKT based method is also invariant.\\

\section{Proof of proposition \ref{lemma:trivial closere}} 
We prove here the extended version of Proposition \ref{lemma:trivial closere}.
\begin{proposition}
 Let $\mathcal{L}(\Vec{x} ;\theta):\reals^d \times \reals^p \rightarrow \reals $ be $G$-invariant function w.r.t $\Vec{x}$. Consider the following optimization problem
\begin{equation}
        min_{x\in \reals^d} \mathcal{L}(x, \theta)
    \end{equation}
solved by the following iterates of GD with some learning rates $\eta_t$  \begin{equation}
        x_t = x_{t-1} -\eta_t \nabla_x \mathcal{L}(x_{t-1}, \theta),~~t = 1,2,\dots, T
    \end{equation}
    Then
    \begin{enumerate}
        \item  The gradient step is G-equivariant function of $x_{t-1}$. 
        \item $Stab_G(x_{t-1}) \subseteq Stab_G(x_{t})$
    \end{enumerate}
\end{proposition}\label{appendix: proof for lemma 1}
We first start by addressing the equivariance of the gradient by $x$.
\begin{lemma} \label{lemma: equivariant gradient}
Let $\mathcal{L}(x;\theta):\reals^d \times \reals^p \rightarrow \reals $ be $G$-invariant function w.r.t $x$ and assume $\rho(g)$ is an orthogonal matrix for all $g\in G$. If $\mathcal{L}$ is derivable by $x$ at $(x_0 , \theta_0)$  then $\mathcal{L}(x;\theta)$ is also derivable by $x$ at $\left\{(g\cdot x_0 , \theta_0)\right\}_{g \in G}$ .
\\ Moreover, $  \nabla_x \mathcal{L}(g \cdot x;\theta) = g \nabla_x \mathcal{L}(x;\theta) , \forall g \in G$.
\end{lemma}
\begin{proof}
For convenience, we would write $\mathcal{L}(\cdot)$ instead of $\mathcal{L}(\cdot , \theta_0)$ \\
By definition, for any direction $h \in \reals^d , ||h||_2 = 1$ ,\begin{equation}
     D_h f(x) =  <\nabla f(x) , h> 
 \end{equation}
 Where $D_h f(x)$ is the directional derivative of f at $x$. \\
  $f$ is $G$-invariant, therefore: 
 \begin{align}
        & D_{g^{-1}\cdot h} \mathcal{L}(x) = \lim_{\epsilon \rightarrow 0} \frac{\mathcal{L}(x+\epsilon g^{-1}\cdot h) - \mathcal{L}(x)}{\epsilon} \\
        & = \lim_{\epsilon \rightarrow 0} \frac{\mathcal{L}( g \cdot x+\epsilon g \cdot g^{-1} \cdot h) - \mathcal{L}(g \cdot x)}{\epsilon} \\
        & = lim_{\epsilon \rightarrow 0} \frac{\mathcal{L}( g\cdot x+\epsilon \cdot h) - \mathcal{L}(g \cdot x)}{\epsilon} \\
        & =  D_{h} \mathcal{L}(g \cdot x)
\end{align}
 On one hand,\begin{equation}
        D_{g^{-1}\cdot h} \mathcal{L}(x)   = <\nabla \mathcal{L}(x) , g^{-1} h> 
 \end{equation}
 On the other hand $g^{-1}$ is orthogonal, then
 \begin{align} 
     & D_{g^{-1}\cdot h} \mathcal{L}(x) = D_{h} \mathcal{L}(g \cdot x) \\
     & = <\nabla \mathcal{L}(g\cdot x) ,h> \\
     & = < g^{-1} \nabla \mathcal{L}(g\cdot x) ,\cdot g^{-1} \cdot h> \\
\end{align}
Therefore,\begin{equation}
   \nabla \mathcal{L}(x)  = g^{-1} \nabla \mathcal{L}(g\cdot x)
 \end{equation}
\end{proof}
 
In other words $\nabla_\theta \mathcal{L}(\cdot, \theta)$ is $G$-equivariant.\\
$\mathcal{L}$ is $G$-invariant, then by \lemmaref{lemma: equivariant gradient}, for every $g\in G$ and for any $x_{t-1} \in \reals^d$:\begin{equation}
        g x_{t-1} -\eta_t \nabla_x \mathcal{L}(g x_{t-1}, \theta)  = g x_{t-1} -g \eta_t \nabla_x \mathcal{L}(x_{t-1}), \theta) = g \cdot x_t
    \end{equation}
    Therefore the gradient step is $G$-equivariant.\\
if $g \in Stab(x_{t-1})$, then by definition $g x_{t-1} = x_{t-1}.$ Therefore:
\begin{align}
        & g\cdot x_t = g\cdot (x_{t-1} -\eta_t \nabla \mathcal{L}(x_{t-1}, \theta) )\\
        & = g\cdot x_{t-1} -\eta_t g\cdot \nabla \mathcal{L}(x_{t-1}, \theta) \\
        & = x_{t-1} - \eta_t \nabla \mathcal{L}(g \cdot x_{t-1}, \theta) \\
        & = x_{t-1} - \eta_t \nabla \mathcal{L}(x_{t-1}, \theta) \\
        & = x_{t}
\end{align} 
Therefore $g\in Stab(x_t)$

\section{Experimental setting} \label{appen: Experimental setting}

\textbf{Setting.}
We focus on image data and considered 4 groups for our experiments - the trivial group, group of 2 elements acting as horizontal reflection, the group of 4 elements acting as horizontal and vertical reflection $G_4$ (Klein four-group), and the Dihedral group $D_4$ (rotations and reflections). To construct the invariant model we used a ReLU neural network with two hidden layers of width 1000 each and applied symmetrization \footnote{ symmetrization is a common practice to project functions on the invariant function space using Reynolds operator $\phi(x;\theta) = \frac{1}{|G|} \sum_{g\in G} \tilde{\phi}(gx;\theta)$ }. Initially we trained the models on MNIST images with binary labels for odd or even digits. For each group we trained a neural network with different training set sizes $n= \{ 10, 20, 50, 100,200 \}$ for 100K epochs. All training ended with $\sim 1e-6$ training error and $100\%$ accuracy. For each method and configuration, we ran 5 experiments of reconstruction with different seeds and $m=1000$ candidates (500 per class).

\begin{figure}
\floatconts
  {fig: dssim}
  {\caption{The mean DSSIM over MNIST training subsets with varying size and different groups.}}
  {%
    \subfigure[KKT]{\label{fig:image-a}%
      \includegraphics[width=0.6\linewidth]{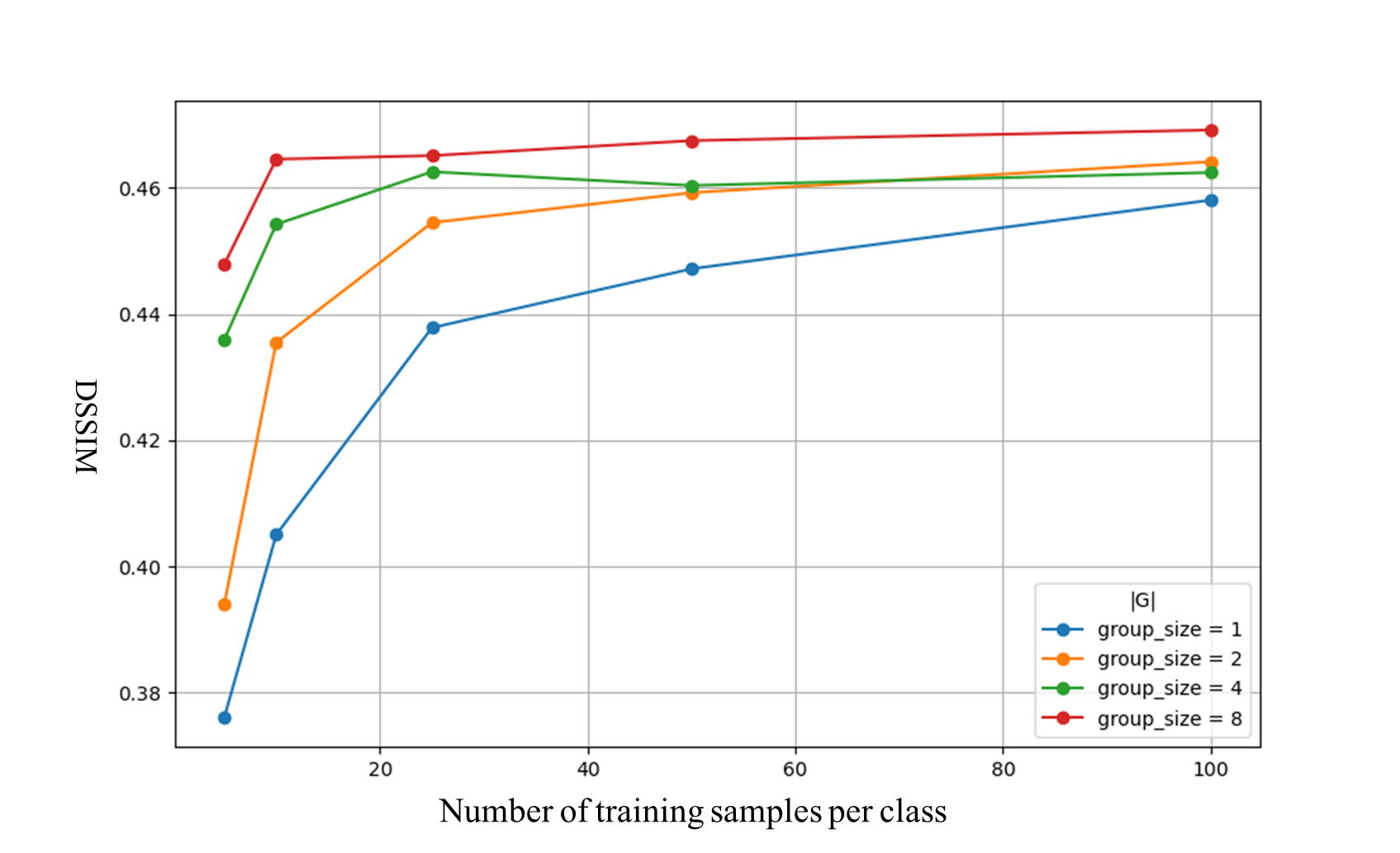}}%
    \subfigure[AM]{\label{fig:image-b}%
      \includegraphics[width=0.6\linewidth]{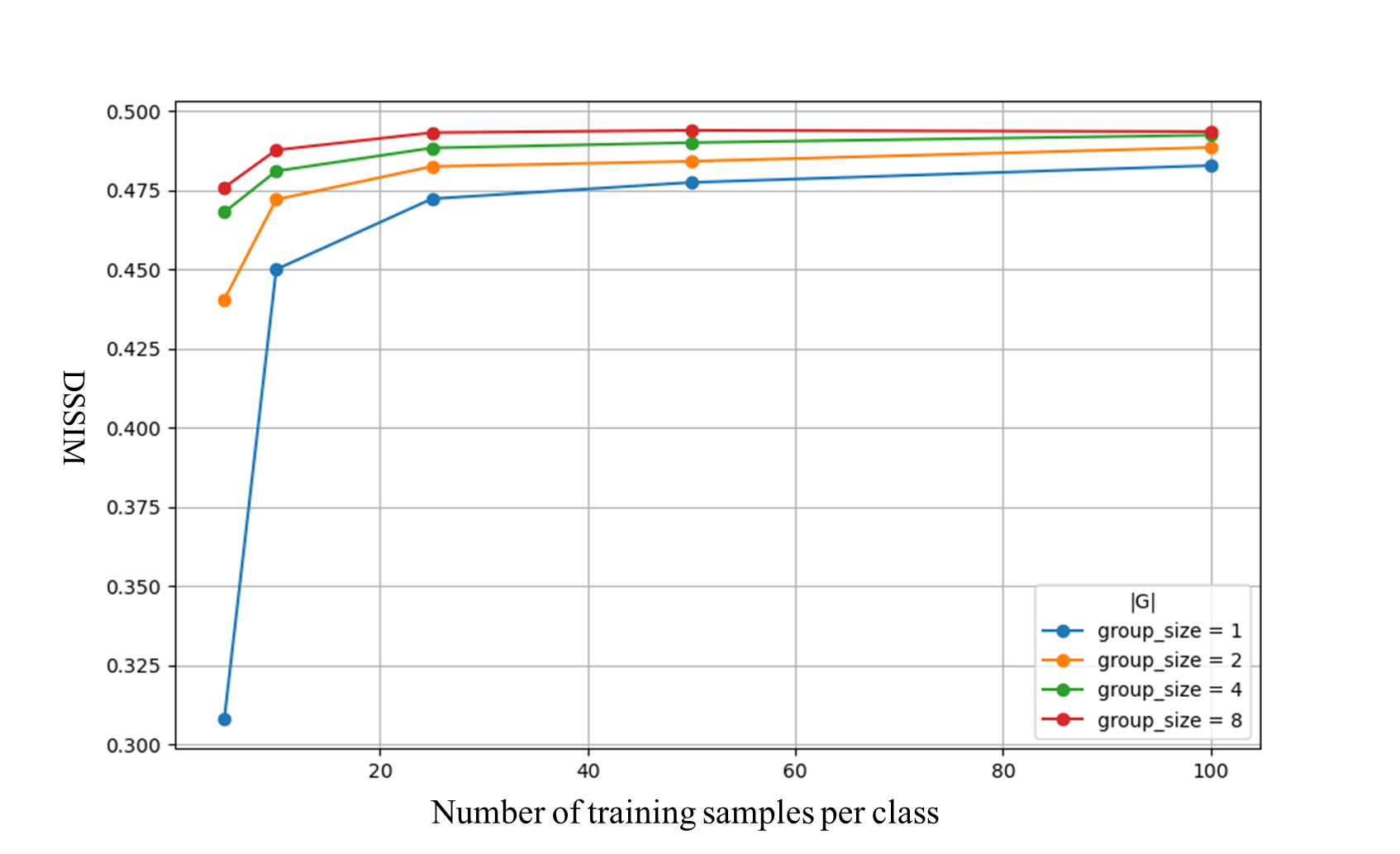}}
  }
\end{figure}

\begin{figure}[htbp]
\floatconts
  {fig: results full}
  {\caption{Pairs of training samples and their corresponding nearest neighbors reconstructions on their right , where $n=50 ,|G|=2$.}}
  {%
    \subfigure[Activation Maximization]{\label{fig:am results}%
      \includegraphics[width=0.75\linewidth]{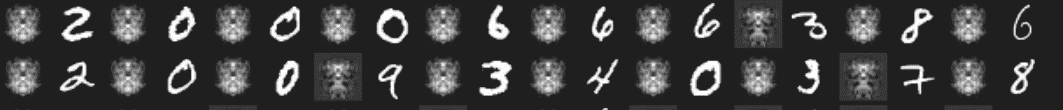}}
    \\
    \subfigure[Vanilla KKT]{\label{fig:kkt results}%
      \includegraphics[width=0.75\linewidth]{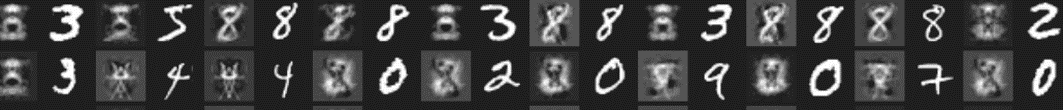}}%
    \\
    \subfigure[KKT with SAME-GD]{\label{ran:sam results}%
      \includegraphics[width=0.75\linewidth]{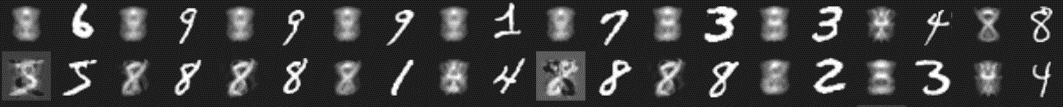}}
    \\
    \subfigure[KKT with Deep Image Prior]{\label{fig:dip results}%
      \includegraphics[width=0.75\linewidth]{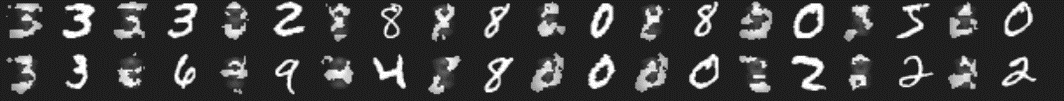}}
  }
\end{figure}

\begin{figure}
\label{fig: histograms}
\floatconts
  {fig:dssim kkt}
  {\caption{The empirical distribution of reconstructions across orbitopes using KKT-based method on MNIST-trained invariant networks. Orbitopes are discretized into bins, each representing a convex combination of orbit elements. Reconstructions are assigned to bins based on their nearest neighbor in the discretized orbitope.}}
  {%
    \subfigure[]{\label{fig:image-a}%
      \includegraphics[width=0.299\linewidth]{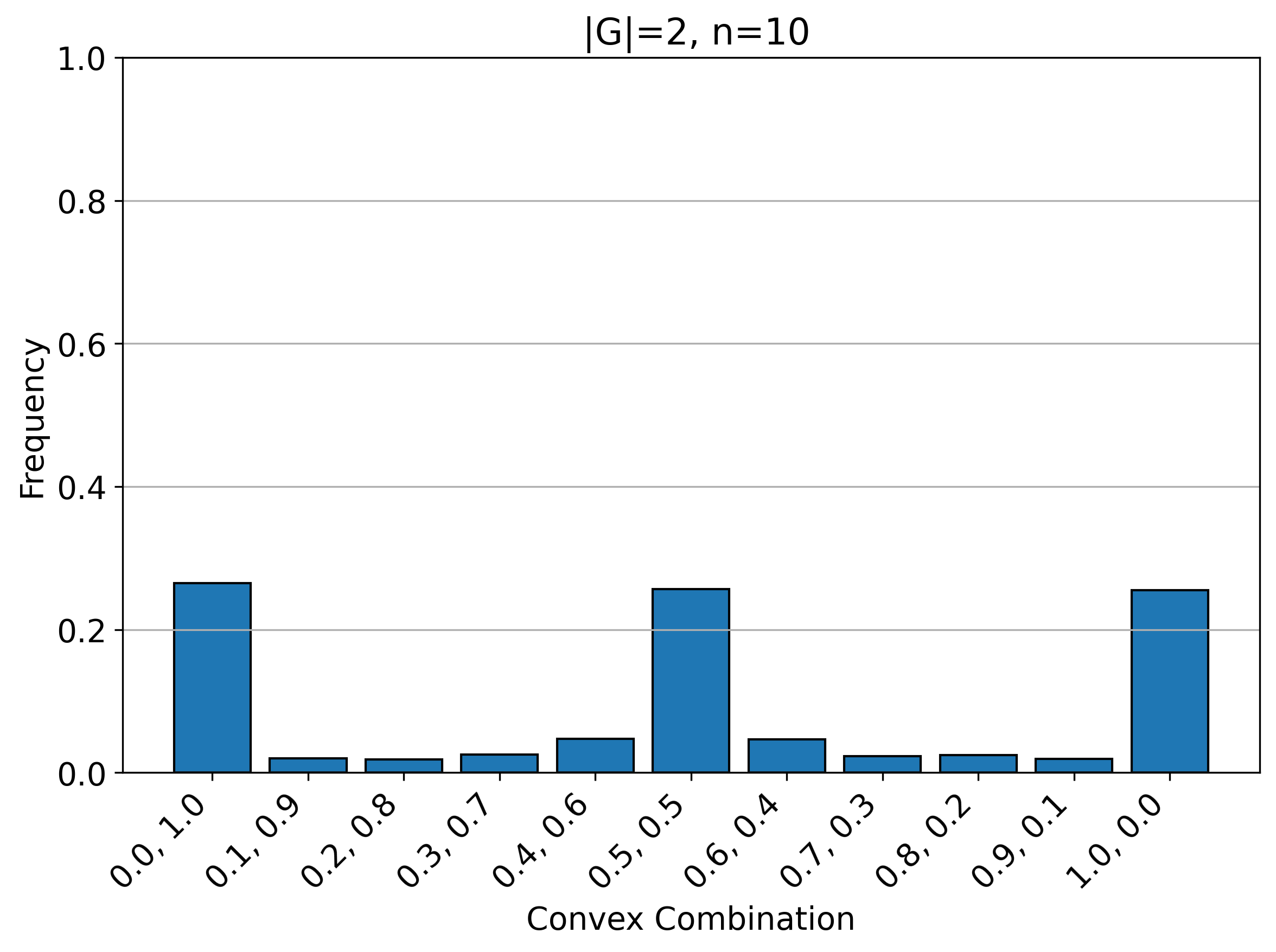}}%
    \qquad
    \subfigure[]{\label{fig:image-b}%
      \includegraphics[width=0.299\linewidth]{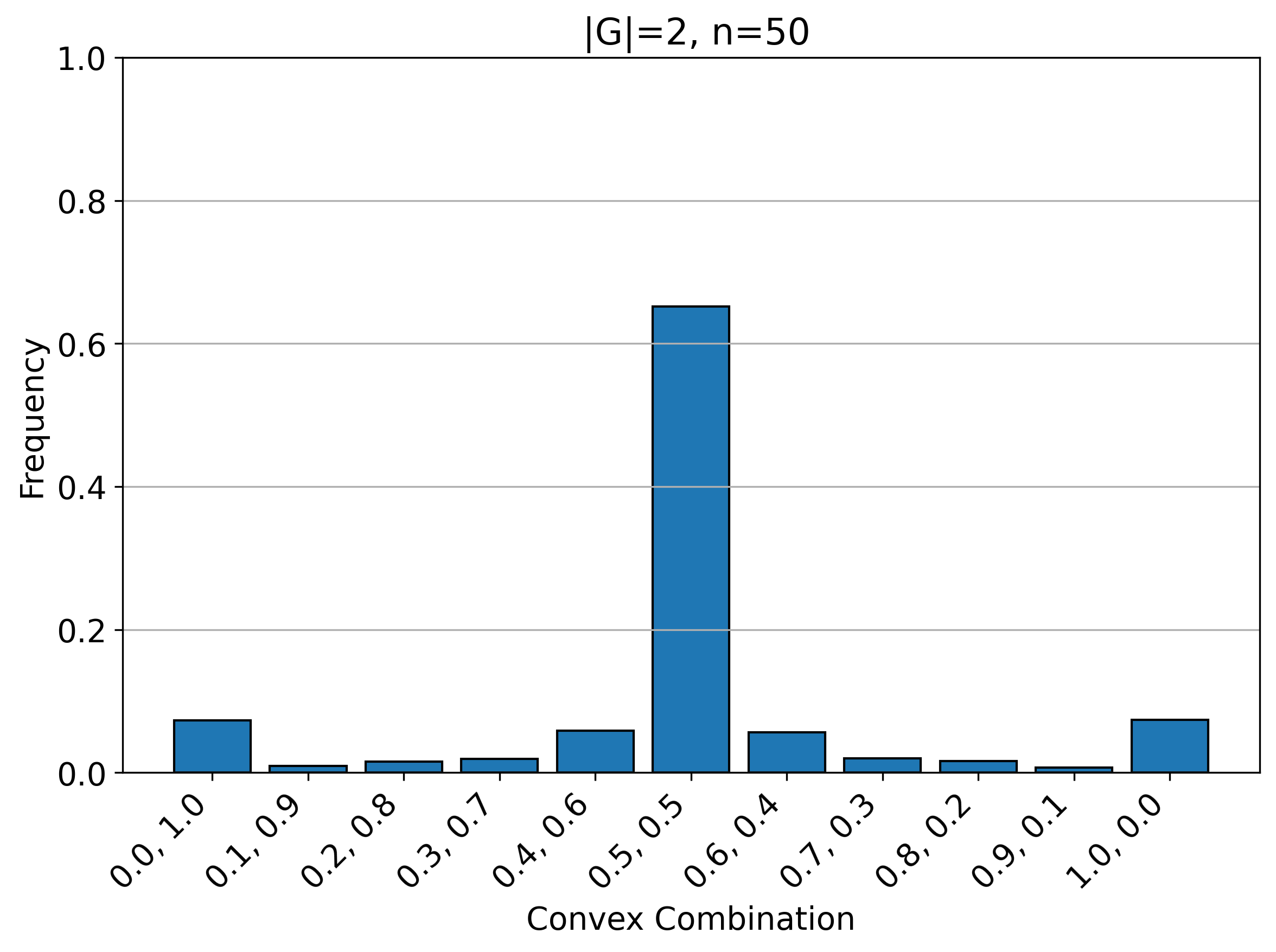}}%
    \qquad
    \subfigure[]{\label{fig:image-c}%
      \includegraphics[width=0.299\linewidth]{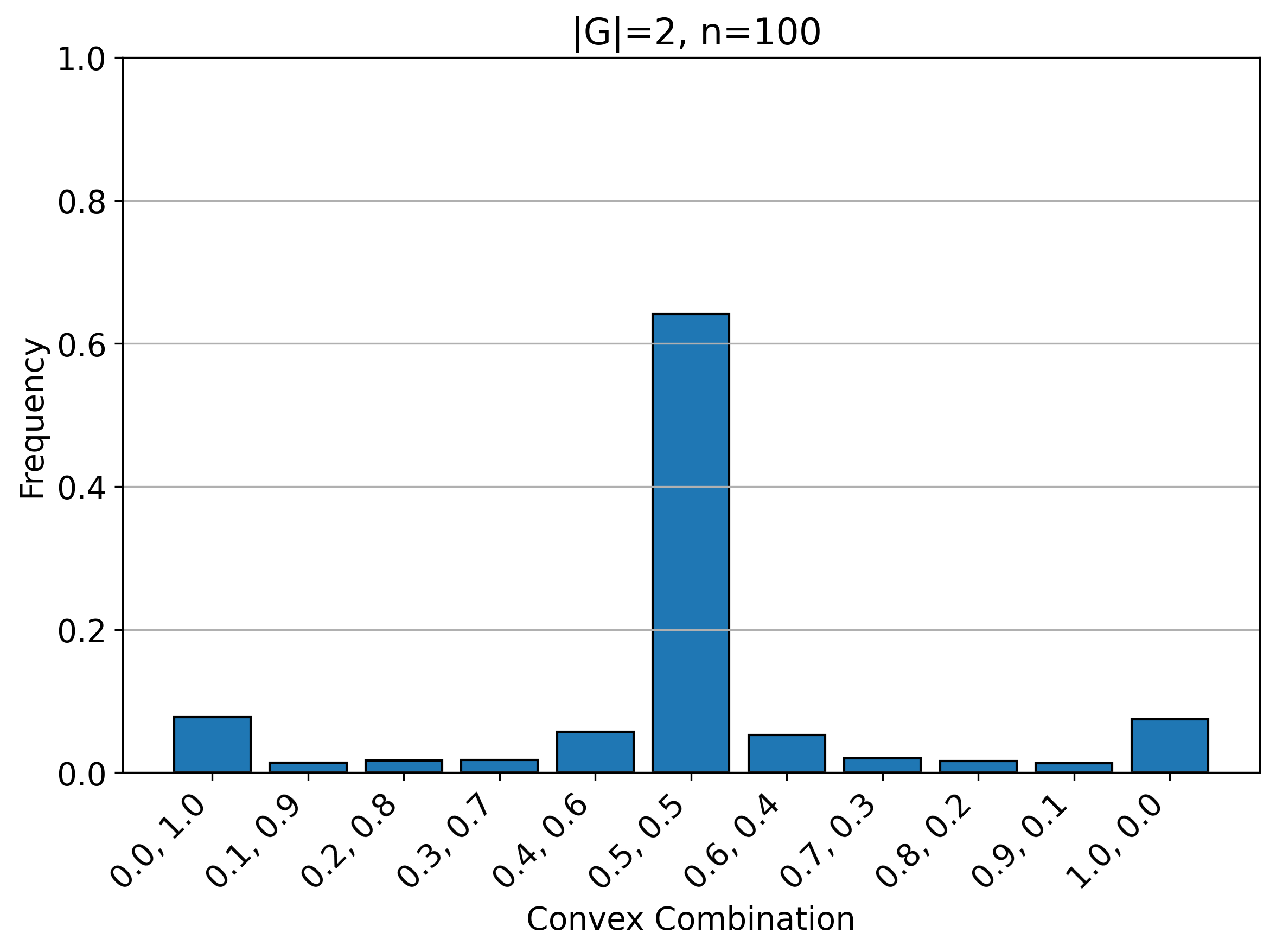}}%
    \\
        \subfigure[]{\label{fig:image-d}%
      \includegraphics[width=0.299\linewidth]{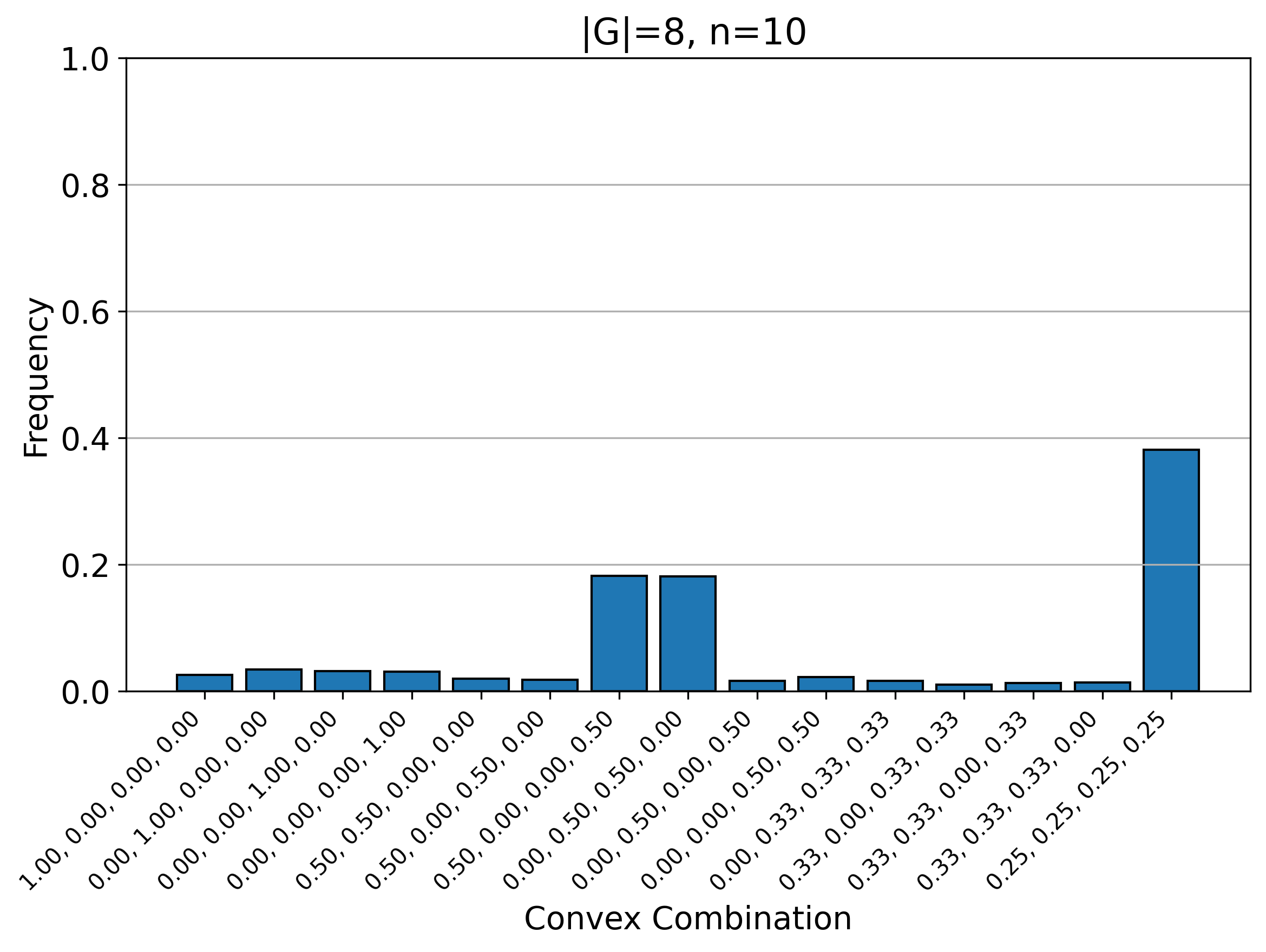}}%
    \qquad
    \subfigure[]{\label{fig:image-e}%
      \includegraphics[width=0.299\linewidth]{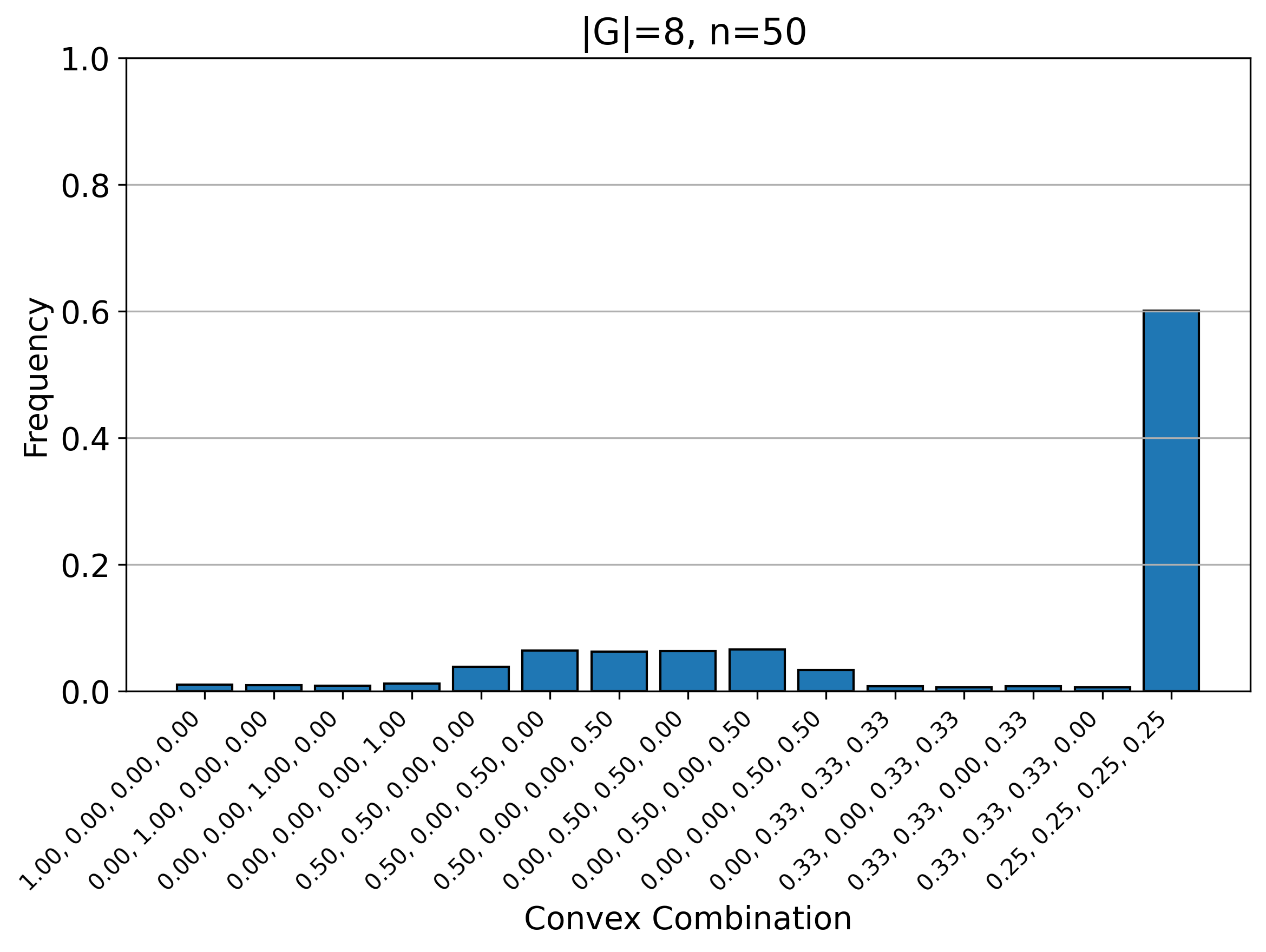}}%
    \qquad
    \subfigure[]{\label{fig:image-f}%
      \includegraphics[width=0.299\linewidth]{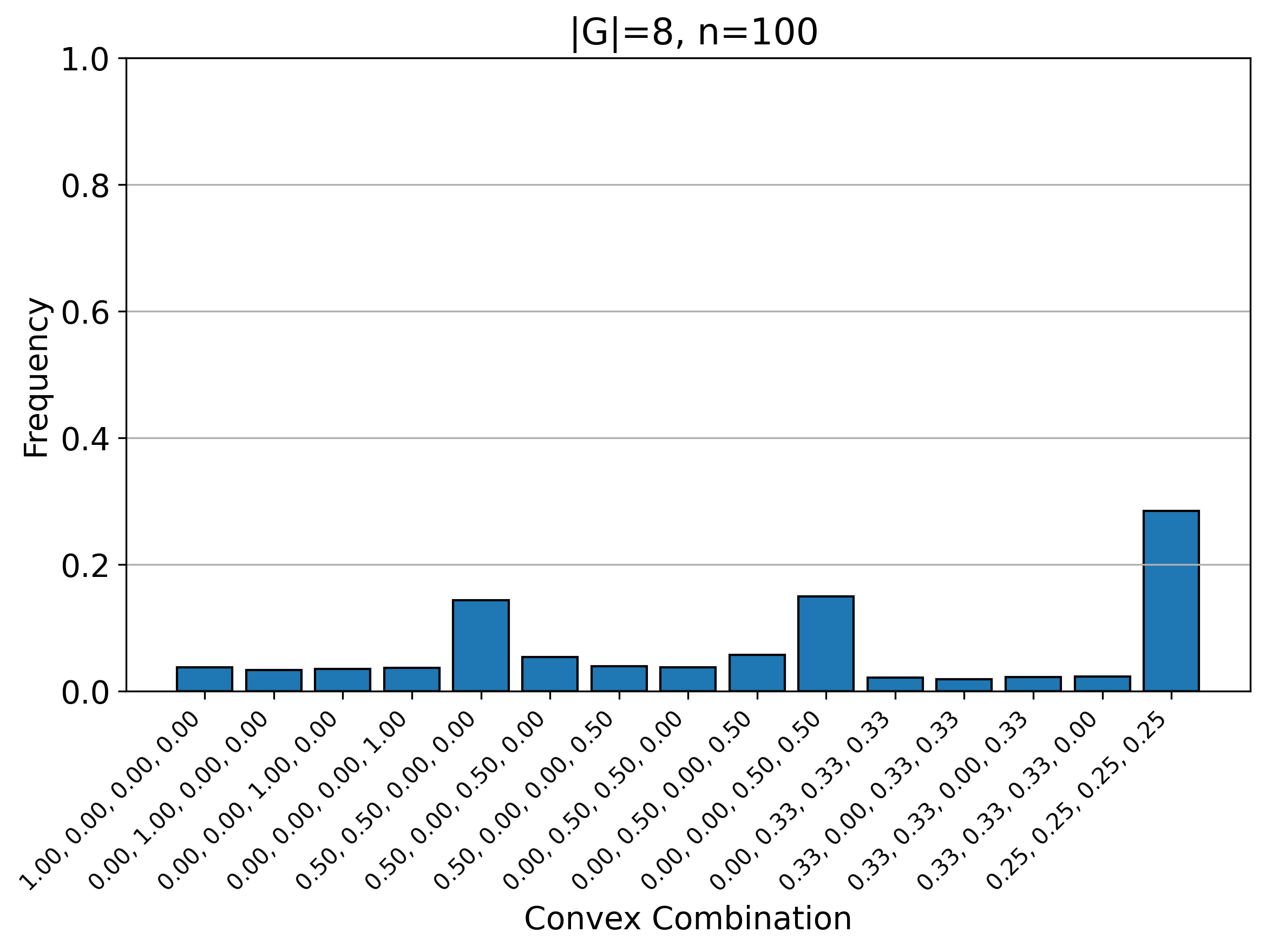}}%
    
  }
\end{figure}

\section{SAME-GD}
\begin{algorithm2e}[H]
\caption{SAME-GD}
\label{alg:sam}
\KwIn{$\{\eta_t , \alpha_t , \beta_t \}_{t=1}^T , T_{save} , T_{update}$}
\KwOut{$x_T$}
Draw $x_0$ \; 
$x_{prev}\leftarrow x_0$\;
\For{$i\leftarrow 1$ \KwTo $T$}{
\If{ $t \% T_{update} \neq 0 $ }{
$x_{t} = x_{t-1} - \eta_t \nabla \mathcal{L}(x_{t-1} ) $\;
}
\Else{  $x_t \leftarrow \alpha_t x_t + (1-\alpha_t) (x_{prev} - \bar{g} x_{prev}) $ \;
}
\If{$t \% T_{save} == 0 $ }{$x_{prev} \leftarrow  \beta_t x_t + (1-\beta_t) \nabla \mathcal{L}(x_{t})^2 $  \;} 
}

\end{algorithm2e}

\section{Discussion and future directions}\label{appendix:future directions}

In this section we discuss in details some challenges and future directions that arise from this work:

\begin{itemize}
    \item Our work focuses on relatively small groups, containing at most $8$ elements, which contain only rotation and reflection transformations. It would be interesting to further study reconstruction from models that are invariant to much larger groups.
    \item Our results indicate that reconstructions from invariant models lie close to the orbitope, mostly to the average over group elements. This finding only scratches the surface regarding how the reconstructions are distributed inside the orbitope, which may be effected by different factors such as initialization, architecture of the network, and structure of the group.
    \item We propose several methods to guide the reconstructions towards specific elements in the group, and thus to reconstruct the actual training samples (up to group action). However, it is not clear how well these methods generalize to larger datasets or larger groups. It will also be interesting to find new methods for this task, for example methods that take advantage of the geometrical properties of the orbitope, and may push the reconstructions towards extreme points of the orbitope.
    \item Our work focuses on image datasets, namely MNIST and CIFAR. It would be interesting to use these methods to reconstruct training samples from different data modalities, such as graphs or point clouds.
    \item In this work the models are constructed to be invariant using symmetrization of the feed-forward model. There are other methods to construct invariant networks, such as parameter sharing-based techniques, and it would be interesting to study data reconstruction attacks on such networks.
\end{itemize}

\end{document}